%% file: ijcai21.tex
\documentclass{article}
\usepackage{ijcai21}

\usepackage{times}
\usepackage{soul}
\usepackage{url}
\usepackage[hidelinks]{hyperref}
\usepackage[utf8]{inputenc}
\usepackage[small]{caption}
\usepackage{graphicx}

\usepackage{grffile}

\usepackage{amsfonts,amsthm,amsmath}
\usepackage{braket}
\usepackage{booktabs}
\usepackage{multirow}
\usepackage{xspace}
\usepackage[linesnumbered,ruled,vlined]{algorithm2e}

\newtheorem{theorem}{Theorem}

\newcommand{\notationFont}{\mathit}
\newcommand{\notationSpace}{\mathcal}
\newcommand{\notationVec}{\mathbf}
\newcommand{\notationReplace}{\text}

\newcommand{\stateSet}{\notationSpace{S}}
\newcommand{\actionSet}{\notationSpace{A}}
\newcommand{\state}{\notationFont{s}}
\newcommand{\action}{\notationFont{a}}

\newcommand{\reward}{\notationFont{r}}
\newcommand{\policy}{\notationFont{\pi}}

\newcommand{\policyb}{\notationFont{\pi_b}}
\newcommand{\policyc}{\notationFont{\hat{\pi}_b}}
\newcommand{\dataset}{\notationSpace{D}}
\newcommand{\actionVec}{\notationVec{a}}
\newcommand{\stateVec}{\notationVec{s}}

\newcommand{\distance}{\notationFont{l}}
\newcommand{\loss}{\notationSpace{L}}

\newcommand{\bellman}{\notationSpace{T}}

\newcommand{\EXP}{\mathbb{E}}
\newcommand{\gmodel}{\mathit{G_\omega(s)}}

\newtheorem{definition}{Definition}
\newtheorem{lemma}{Lemma}

\newcommand{\qvalue}{\notationReplace{state-action value function }}
\newcommand{\error}{\epsilon}
\newcommand{\errorg}{\epsilon_g}

\newcommand{\cheetah}{\textsc{Halfcheetah}}
\newcommand{\hopper}{\textsc{Hopper}}
\newcommand{\walker}{\textsc{Walker2d}}

\newcommand{\methodFont}{\textsl}
\newcommand{\ours}{\methodFont{AQL}\xspace}

\newcommand{\dqn}{\methodFont{DQN}\xspace}
\newcommand{\bc}{\methodFont{BC}\xspace}
\newcommand{\bcq}{\methodFont{BCQ}\xspace}
\newcommand{\ddpg}{\methodFont{DDPG}\xspace}
\newcommand{\bear}{\methodFont{BEAR}\xspace}
\newcommand{\sac}{\methodFont{SAC}\xspace}
\newcommand{\brac}{\methodFont{BRAC}\xspace}
\newcommand{\bracv}{\methodFont{BRAC-v}\xspace}
\newcommand{\bracp}{\methodFont{BRAC-p}\xspace}
\newcommand{\cql}{\methodFont{CQL}($\mathcal{H}$)\xspace}

\newcommand{\ie}{\textit{i.e.,}\xspace}
\newcommand{\eg}{\textit{e.g.,}\xspace}
\newcommand{\etc}{\textit{etc}\xspace}

\let\oldnl\nl
\newcommand{\nonl}{\renewcommand{\nl}{\let\nl\oldnl}}

\title{Boosting Offline Reinforcement Learning with Residual Generative Modeling}

\author{
Hua Wei$^1$
\and
Deheng Ye$^1$\and
Zhao Liu$^1$\and
Hao Wu$^1$\and
Bo Yuan$^1$\and \\
Qiang Fu$^1$\and
Wei Yang$^1$\And
Zhenhui (Jessie) Li $^2$\\
\affiliations
$^1$Tencent AI Lab, Shenzhen, China\\
$^2$The Pennsylvania State University, University Park, USA\\
\emails
\{harriswei, dericye, ricardoliu, alberthwu, jerryyuan, leonfu, willyang\}@tencent.com \\
jessieli@ist.psu.edu
}

\begin{document}

\maketitle

\input{abstract.tex}

\input{intro.tex}
\input{relatedwork.tex}
\input{preliminary.tex}

\input{error.tex}

\input{method.tex}

\input{experiment.tex}
\input{case.tex}

\input{conclusion.tex}

\bibliographystyle{named}
\bibliography{proc}

\input{appendix}

\end{document}

%% file: abstract.tex
\begin{abstract}
Offline reinforcement learning (RL) tries to learn the near-optimal policy with recorded offline experience without online exploration.
Current offline RL research includes: 1) generative modeling, i.e., approximating a policy using fixed data; and 2) learning the state-action value function. 
While most research focuses on the state-action function part through reducing the bootstrapping error in value function approximation induced by the distribution shift of training data, the effects of error propagation in generative modeling have been neglected. 
In this paper, we analyze the error in generative modeling. We propose AQL (action-conditioned Q-learning), a residual generative model to reduce policy approximation error for offline RL. 
We show that our method can learn more accurate policy approximations in different benchmark datasets. 
In addition, we show that the proposed offline RL method can learn more competitive AI agents in complex control tasks under the multiplayer online battle arena (MOBA) game \textit{Honor of Kings}.
\end{abstract}

%% file: intro.tex
\section{Introduction}

Reinforcement learning (RL) has achieved promising results in many domains~\cite{mnih2015human,silver2017mastering,kalashnikov2018scalable,berner2019dota,vinyals2019alphastar,ye2020supervised,ye2020towards}.
However, being fettered by the online and trial-and-error nature, applying RL in real-world cases is difficult~\cite{dulac2019challenges}. 
Unlike supervised learning which directly benefits from large offline datasets like ImageNet~\cite{deng2009imagenet}, current RL has not made full use of offline data~\cite{levine2020offline}. 
In many real-world scenarios, online exploration of the environment may be unsafe or expensive. 
For example, in recommendation~\cite{li2011unbiased} or healthcare domains~\cite{gottesman2019guidelines}, a new policy may only be deployed at a low frequency after extensive testing. In these cases, the offline dataset is often large, potentially consisting of years of logged experiences.
Even in applications where online exploration is feasible, offline data is still beneficial. For example, in strategy video games (\textit{Dota} or \textit{Honor of Kings}) that require highly complex action control, current RL methods still learn from the scratch, resulting in a long time to master human-level skills~\cite{berner2019dota,ye2020mastering}, where there is a large amount of logged replay data from vast players to be utilized.

Although off-policy RL methods~\cite{mnih2015human,lillicrap2015continuous,haarnoja2018soft} may be executed in the offline scenario directly, their performance was shown to be poor~\cite{fujimoto2019off} with a fixed batch of data and no interactions with the environment. 
The poor performance is suspected due to incorrect value estimation of actions outside the training data distribution \cite{kumar2019stabilizing}. The error of estimated values will accumulate and amplify during the Bellman backup process in RL. Typical off-policy RL would improve the learned policy by trying out the policy in the environment to correct the erroneous estimations, which is not applicable in purely offline scenarios.

Instead of correcting erroneous estimations by interacting with the environment, offline RL provides an alternative for the offline scenario. Current representative offline RL methods~\cite{fujimoto2019off,kumar2019stabilizing,wu2019behavior,kumar2020conservative} mainly study how to reduce the error with conservative estimations, \ie constraining the action or state-action distribution around the given dataset's distribution when learning value functions. With such estimations, the learned RL policy can approach or exceed the original policy (also called \textit{behavioral policy}) that generates the fixed dataset. As most offline RL methods assume that the behavioral policy is unknown, typical offline RL methods would first approximate the behavioral policy through generative modeling~\cite{levine2020offline}, \ie learn to output actions for given states. Then, the objective is to learn approximations for selecting the highest valued actions that are similar to the approximated behavioral policy.

Different from the aforementioned methods, this paper focuses on investigating the impact of errors from generative modeling on offline RL. As we will show, in value-based off-policy scenarios, the error of generative modeling will accumulate in the process of learning the Q-value function during Bellman backup. Our main contribution involves studying the accumulating process of generative modeling error in offline RL and developing a practical method to mitigate this error. 
To expand, we first examine the error accumulation process during offline RL learning in detail and analyze the influence of generative modeling error on the final offline RL error on the Q-value function. 
Then we propose an error reduction method via residual learning~\cite{huang2017learning}. 
Through experiments on a set of benchmark datasets, we verify the effectiveness of our method in boosting offline RL performance over state-of-the-art methods. 
Furthermore, in the scenario of the multiplayer online battle arena (MOBA) game \textit{Honor of Kings}, which involves large state-action space, our proposed method can also achieve excellent performance. 

%% file: relatedwork.tex
\section{Related Work}
Under fully offline settings where no additional online data collection is performed, both offline RL methods and imitation learning methods can be used to learn a policy from pre-recorded trajectories. 

\paragraph{Offline RL.} 
Offline RL describes the setting in which a learner has access to only a fixed dataset of experience, in contrast to online RL which allows interactions with the environment.
Existing offline RL methods suffer from issues pertaining to OOD actions~\cite{fujimoto2019off,kumar2019stabilizing,wu2019behavior,kumar2020conservative}. Prior works aim to make conservative value function estimations around the given dataset's distribution, and then only use action sampled from this constrained policy in Bellman backups for applying value penalty. Different from these works, this paper focuses on how the errors propagate throughout the whole process of offline RL, from generative modeling to value approximation.

\paragraph{Imitation learning.} 
Imitation learning (IL) is to learn behavior policies from demonstration data~\cite{schaal1999imitation,hussein2017imitation}. 
Though effective, these methods are not suitable for offline RL setting because they require either on-policy data collection or oracle policy. Different from offline RL, imitation learning methods do not necessarily consider modeling the long-term values of actions or states like reinforcement learning methods.

%% file: preliminary.tex
\section{Preliminaries}
\label{sec:prelim}
\subsection{Reinforcement Learning} We consider the environment as a Markov Decision Process (MDP) $( \stateSet, \actionSet, P, R, \gamma )$, where $\stateSet$ is the state space, $\actionSet$ is the action space, $P:\mathcal{S}\times\mathcal{A}\times\mathcal{S}\rightarrow[0,1]$ denotes the state transition probability, 
$R:\mathcal{S}\times\mathcal{A}\rightarrow\mathbb{R}$ represents the reward function, and $\gamma\in(0, 1]$ is the discount factor. 
A policy $\pi$ 
is a mapping $\mathcal{S}\times\mathcal{A}\rightarrow [0, 1]$.

A value function provides an estimate of the expected cumulative reward that will be obtained by following some policy $\policy(\actionVec_t | \stateVec_t)$ 
when starting from a state-action tuple $(\state_t, \action_t)$ in the case of the state-action value function $Q^\policy(\state_t, \action_t)$: 

\begin{equation}
\label{eq:bellman-q}
\begin{aligned}
Q^\policy(\state_t, \action_t) & =r(\state_t, \action_t) + \gamma \EXP_{\state_{t+1}, \action_t}
[Q^\policy(\state_{t+1}, \action_{t+1})]
\end{aligned}
\end{equation}

In the classic \textit{off-policy} setting, the learning of Q-function is based on the agent's replay buffer $\dataset$ that gathers the experience of the agent in the form of $(\state_t, \action_t, \reward_t, \state_{t+1})$, and each new policy $\policy_k$ collects additional data by exploring the environment. 
Then $\dataset$, which consists of the samples from $\policy_0, \dots,\policy_k$, is used to train a new updated policy $\policy_{k+1}$.

\subsection{Offline Reinforcement Learning} Offline RL additionally considers the problem of learning policy $\pi$ from a fixed dataset $\dataset$ consisting of single-step transitions $(\state_t, \action_t, \reward_t, \state_{t+1})$, without interactions with the environment. This is in contrast to many off-policy RL algorithms that assume further interactions with the current policy rolling out in the environment. In this paper, we define the behavioral policy $\policyb$ as the conditional distribution $p(a|s)$ in the dataset $\dataset$, which is treated unknown. In real-world cases, the training dataset could be generated by a collection of policies. For simplicity, we refer to them as a single behavioral policy $\policyb$.

\subsubsection{Value Function Approximation} 
For large or continuous state and action spaces, the value function can be approximated with neural networks $\hat{Q}_\theta$, parameterized by $\theta$.  With the notion of Bellman operator $\bellman^\policy$, we can denote Equation~\eqref{eq:bellman-q} as $\hat{Q}^\policy = \bellman^\policy \hat{Q}^\policy $ with $\gamma \in [0,1)$. This Bellman operator has a unique fixed point that corresponds to the true state-value function for $\policy(\action|\state)$, which can be obtained by repeating the iteration $\hat{Q}^\policy_{k+1} = \bellman^\policy_{k} \hat{Q}^\policy$, and it can be shown that $\lim_{k\rightarrow\infty} \hat{Q}^\policy_{k} = \bellman^\policy \hat{Q}^\policy$. 

Offline RL algorithms based on this value function approximation with the iterative update are shown to suffer from action distribution shift~\cite{fujimoto2019off} during training. Since the policy in next iterate is computed by choosing actions that greedily maximize the \qvalue at each state, $\policy_{k+1}(\action|\state)=\arg\max_{\action}Q^\policy_{k}(\state,\action)$, it may be biased towards out-of-distribution (OOD) actions with erroneously high Q-values. In RL with explorations, such errors can be corrected by rolling out the action in the environment and observing its true value. In contrast, an offline RL agent is unable to query such information from the environment. To mitigate the problem caused by OOD actions, typical offline RL methods focus on constraining the learned policy to output actions that lie in the training distribution $\policyb$.

\subsubsection{Generative Modeling}  Because we do not assume direct access to $\policyb$, it is common in previous work to approximate this behavior policy via a generative model $\gmodel$, with max-likelihood over $\dataset$~\cite{fujimoto2019off,kumar2019stabilizing,wu2019behavior,levine2020offline}:

\begin{equation}
\gmodel = \hat{\policy}_b:= \arg \max_{\hat{\policy}} \EXP_{(\state, \action, \reward, \state')\sim \dataset}[\log \hat{\policy}(a|s)]
\end{equation}

We denote the approximated policy as $\hat{\policyb}$ and refer to it as ``cloned policy'' to distinguish it from $\policyb$.

In Section~\ref{sec:gen-model-err}, we will analyze how the errors propagate from generative modeling to value function approximation, resulting in the overall errors for offline RL. Then we will introduce how to reduce the overall errors of offline RL by reducing the generative modeling error, with theoretical proofs and implemented models in Section~\ref{sec:residual-gen}.

%% file: error.tex
\section{Generative Modeling Error Propagation}
\label{sec:gen-model-err}
In this section, we define and analyze how the generative modeling error propagates during the process of value estimation in offline RL. We derive bounds which depend on the generative modeling. This motivates further focusing on mitigating generative modeling error. 

\subsection{Generative Modeling Error}
As discussed in the last section, we need to approximate $\policy_\dataset(s,a)$ with a generative model $\gmodel$. If we train $\gmodel$ with supervised learning (i.e., standard likelihood maximization) on $\dataset$, we have the following result from~\cite{ross2011reduction}, whose proof can be found in~\cite{levine2020offline}.

\begin{lemma}[Behavioral cloning error bound]
\label{lemma:bc-errorbound}
If $\tau^{\policyb}(s)$ is the state distribution induced by $\policyb$  and $\policy(a|s)$ is trained via standard likelihood maximization on $\stateVec \sim \tau^{\policyb}(s)$ and optimal labels $\actionVec$, and attains generalization error $\error_g$ on $\stateVec \sim \tau^{\policyb}(s)$, then $\distance(\policy) \leq C + H^2\error_g$ is the best possible bound on the expected error of the learned policy, where $C$ is the true accumulated reward of $\policyb$.
\end{lemma}

This means that even with optimal action labels, we still get an error bound at least quadratic in the time horizon $H$ in the offline case. Intuitively, the policy $\policyc$ learned with generative model $\gmodel$ may enter into states that are far outside of the training distribution, where the generalization error bound $\errorg$ no longer holds on unseen states during training. Once the policy enters one OOD state, it will keep making mistakes and remain OOD for the remainder of the testing phase, accumulating $O(H)$ error. Since there is a non-trivial chance of entering an OOD state at every one of the $H$ steps, the overall error scales as $O(H^2)$.

\subsection{Error Propagation on Value Estimation}

\begin{definition}[Value function approximation error]
We define $\error^\policy_\dataset(s,a)$ as the value function approximation error between the true \qvalue $Q^\policy_\dataset$ computed from the dataset $\dataset$ and the true \qvalue $Q^*$ :

\begin{equation}
    \error^\policy_\dataset(s,a) = Q^*(s,a) - Q^\policy_\dataset(s,a)
\end{equation}
\end{definition}

\begin{theorem}
\label{eq:qvalue-errorbound}
Given a policy $\policyb$ that generates the dataset $\dataset$, if we model its cloned policy $\policyc$ from $\dataset$ with a generative modeling error of $\error_g$, assume that $\delta(s,a) = \sup_\dataset \error^\policy_\dataset(s,a) $ and $\eta(s,a) = \sup_\dataset \errorg(a|s) $, with the action space of dimension $|\actionSet|$, $\delta(s,a)$ satisfies the following:
\begin{equation}
\label{eq:error}
\begin{aligned}
&  \delta(s,a) \geq  \sum\limits_{s'} (p^*(s'|s, a) - p_\dataset(s'|s, a)) \\
    & \cdot [r(s,a)+\gamma(\sum\limits_{a'}(\policyb(a'|s')+ |\actionSet|\eta ) Q^\policy_\dataset(s',a') ] \\
 & + p^*(s'|s, a)\cdot \gamma (\sum\limits_{a'}\policyb(a'|s')+ |\actionSet|\eta ) \error^\policy_\dataset(s,a)
\end{aligned}
\end{equation}

\end{theorem}

\begin{proof}
Firstly, we have the following:
\begin{equation}
\begin{aligned}
    & \error^\policy_\dataset(s,a) =   \sum\limits_{s'} (p^*(s'|s, a) - p_\dataset(s'|s, a)) \\
    & \cdot [r(s,a)+\gamma\sum\limits_{a'}(\policyb(a'|s') + \errorg(a'|s')) Q^\policy_\dataset(s',a') ] \\
 & + p^*(s'|s, a)\cdot \gamma \sum\limits_{a'}(\policyb(a'|s') + \errorg(a'|s'))\error^\policy_\dataset(s',a')
\end{aligned}
\end{equation}

The proof follows by expanding each Q, rearranging terms, simplifying the expression and then representing cloned policy $\policy$ with behavior policy $\policyb$ with a generative error $\errorg$.

Based on above equation, we take the supremum of $\error^\policy_\dataset(s,a)$ and have the following:

\begin{equation}
\label{eq:error-sup}
\begin{aligned}
& \sup_\dataset \error^\policy_\dataset(s,a) \geq \sum\limits_{s'} (p^*(s'|s, a) - p_\dataset(s'|s, a)) \\
    & \cdot [r(s,a)+\gamma(\sum\limits_{a'}(\policyb(a'|s')+ |\actionSet|\eta ) Q^\policy_\dataset(s',a') ] \\
 & + p^*(s'|s, a)\cdot \gamma (\sum\limits_{a'}\policyb(a'|s')+ |\actionSet|\eta ) \error^\policy_\dataset(s,a) \\
\end{aligned}
\end{equation}

\end{proof}

For a finite, deterministic MDP, if all possible state transitions are captured in $\dataset$, $p_\dataset(s'|s, a)$ will be equivalent to $p^*(s'|s, a)$, we will have $\error^\policy_\dataset(s,a)=0$. However, in infinite or stochastic MDP, it might require an infinite number of samples to cover the true distribution. Therefore, $\delta(s,a)=0$ if and only if the following strong assumptions holds: the true MDP is finite and deterministic, and all possible state transitions are captured in $\dataset$. Otherwise, we have $\delta(s,a)>0$.

From Theorem~\ref{eq:qvalue-errorbound}, we have $\delta(s,a)$ scales as $O(\eta|\actionSet|)$, where $\eta(s,a) = \sup_\dataset \errorg(a|s) $.  Intuitively, this means we can prevent an increase in the \qvalue error by learning a generative model $\gmodel$ with smaller $\errorg$. Meanwhile, for settings where the action space is small, the $\errorg$ are will have smaller influences in inducing the \qvalue error. Overall, in the time horizon $H$, since the generative error $\errorg$ scales as $O(H^2)$, $\delta(s,a) = \sup_\dataset \error^\policy_\dataset(s,a)$ scales as $O(|\actionSet| H^2)$.

%% file: method.tex
\section{Residual Generative Modeling}
\label{sec:residual-gen}
We begin by analyzing the theoretical properties of residual generative modeling in a deterministic policy setting, where we are able to measure the monotonically decreasing loss over sampled actions precisely. We then introduce our deep reinforcement learning model in detail, by drawing inspirations from the deterministic analog.	

\subsection{Addressing Generative Modeling Error with Residual Learning}
\label{sec:model-addressing}
In this section, we will provide a short theoretical illustration of how residual learning can be used in addressing the generative modeling under deterministic policies. In this example, we will assume that $ \rho_{\phi}(s) $  is the original generative model in mimicking the policy  $ \policyb $  in dataset  $\dataset$, and the residual generative model is denoted as
$ \hat{a} = G_{\omega}(s,\rho_{\phi}(s)) $, where $\omega$ stands for the parameters used to combine input $s$ and generative model $ \rho_{\phi}(s) $.


Without loss of generality, we can re-write the final output layer by assuming that $ G_{\omega} $  is parameterized by a linear weight vector  $ \mathbf{w} $  and a weight matrix  $ \mathbf{M} $, and previous layers can be represented as $G_{\omega_2}(s)$. Thus, $\hat{a}$ can be denoted as $\hat{a} = \mathbf{w}^{T}(\mathbf{M}G_{\omega_2}(s)+ \rho_{\phi}(s))$. 
We have the following result from~\cite{shamir2018resnets}:

\begin{lemma}
Suppose we have a function defined as
\begin{equation}
    \Gamma_{\psi} (\mathbf{a}, \mathbf{B}) \doteq \Gamma (\mathbf{a}, \mathbf{B}, \psi) \doteq \mathbb{E}_{\mathbf{x}, y} \left[\mathit{l} \left(\mathbf{a}^T \left(H(\mathbf{x}) + \mathbf{B} F_{\psi} (\mathbf{x} \right), y \right) \right] 
\end{equation}
where $\mathit{l}$ is the defined loss, $\mathbf{a}$, $\mathbf{B}$ are weight vector and matrix respectively, and $\psi$ is the parameters of a neural network. Then, every local minimum of $\Gamma$ satisfies
\begin{equation}
    \Gamma (\mathbf{a}, \mathbf{B}, \psi) \leq \inf_{\mathbf{a}} \Gamma (\mathbf{a}, \mathbf{0}, \psi)
\end{equation}
if following conditions are satisfied:
(1) loss $l (\hat{y}, y)$ is twice differentiable and convex in $\hat{y}$;
(2) $\Gamma_{\psi} (\mathbf{a}, \mathbf{B})$, $\bigtriangledown \Gamma_{\psi} (\mathbf{a}, \mathbf{B})$, and $\bigtriangledown^2 \Gamma_{\psi} (\mathbf{a}, \mathbf{B})$ are Lipschitz continuous in ($\mathbf{a}$, $\mathbf{B}$).
\end{lemma}

\begin{theorem}
\label{theo:residual}
When using log loss or squared loss for deterministic policy and linear or convolution layers for $F_{\omega_2}$, every local optimum of  $ \hat{a} = \mathbf{w}^{T}(\mathbf{M}G_{\omega_2}(s)+ \rho_{\phi}(s))  $ will be no worse than  $ \hat{a} = \mathbf{w}^{T}\rho_{\phi}(s)  $ . 
\end{theorem}

\begin{proof}
For deterministic policies, $G_{\omega_2}(\mathbf{x}) $ is Lipschitz continuous when using linear or convolution layers~\cite{virmaux2018lipschitz}. Since log loss and squared loss are twice differentiable in $ a $, we have 
\begin{equation}
\Gamma(\mathbf{w}, M, \theta) \leq \text{inf}_\mathbf{w} \Gamma(\mathbf{w}, 0, \theta) \end{equation}
where $\Gamma_{\psi} (\mathbf{a}, \mathbf{B}) = \Gamma (\mathbf{a}, \mathbf{B}, \psi) \doteq \mathbb{E}_{\mathbf{x}, y} [l (\mathbf{a}^T (H(\mathbf{x}) + \mathbf{B} F_{\theta} (\mathbf{x}) ), y ) ] $,  $ \mathit{l} $  is the defined loss,  $ \mathbf{a} $ , $ \mathbf{B} $  are weight vector and matrix respectively, and  $ \theta $  is the parameters of a neural network.

That is, the action-conditioned model has no spurious local minimum that is above that of the original generative model  $ \rho_{\phi}(s) $. 
\end{proof}

\subsection{Residual Generative Model}
The main difference between our method and existing offline RL methods is that we design a residual modeling part for the generative model when approximating the $\policy_\dataset$. Therefore, in this section, we mainly introduce our approach to offline reinforcement learning, \ours (action-conditioned Q-learning), which uses action-conditioned residual modeling to reduce the generative modeling error.

Our generative model consists of two major components: a conditional variational auto-encoder (VAE) that models the distribution by transforming an underlying latent space, and a residual neural network that models the state-action distribution residuals on the output of the VAE.


\subsubsection{Conditional VAE} 
To model the generative process of predicting actions given certain states, analogous to existing literatures like Batch-Constrained Q-learning (\bcq)~\cite{fujimoto2019off} and Bootstrapping error reduction (\bear)~\cite{kumar2019stabilizing}, we use a conditional VAE that takes state and action as input and outputs the reconstructed action. 
Given the raw state, we first embed the state observation with a state embedding module in conditional VAE. Then in the \textit{encoder} part, we concatenate the state embedding with action input and output the distribution of latent space (assumed to be Gaussian for continuous action space and Categorical for discrete action space). In the \textit{decoder} part, we concatenate the state embedding and the latent variable $z$ from the learned distribution and outputs the reconstructed action. The overall training loss for the conditional VAE is:
\begin{equation}
\loss_{VAE} = - \EXP_{z\sim q(z|\state,\action)}[\log p(\action|\state, z)] + D_{KL}(q(z|\state,\action)||p(z))
\end{equation}
where the first term is the reconstructed loss, and the second term is the regularizer that constrains the latent space distribution, $\state$ is the state input to the conditional VAE, $\action$ is the action in the dataset in pair with $\state$.

\subsubsection{Residual Network}
Unlike \bcq and \bear which take VAE or a single feed-forward neural network as the generative model, we propose to use the reconstructed action from VAE as an additional input to learn the residual of action output. This residual mechanism is motivated to boost offline RL by reducing the generative model's error. The overall loss for the generative network is:

\begin{equation}
\label{eq:loss-total}
\begin{aligned}
\loss_{\omega} = \EXP_{\omega}[d(\hat{\action},\action)] +  \loss_{VAE}
\end{aligned}
\end{equation}
where $\hat{\action}=\gmodel$ is the output of the residual network, and $d(\hat{\action},\action)$ is the distance measure between two actions. For continuous actions, $d$ could be defined as the mean squared error; for discrete actions, $d$ could be defined as the cross-entropy. Intuitively, this loss function includes the original generative modeling loss function (usually treated the same as a behavior cloning loss) and a VAE loss, optimized at the same time.


\subsection{Training Process}
We now describe the practical offline RL method based on \bear, a similar variant on \bcq or other methods can be similarly derived. Empirically we find that our method based on \bear performs better. We've described our generative model in previous sections, here we briefly introduce the other part of the offline RL algorithm, i.e., value function approximation process, which is similar to \bcq and \bear. To compute the target Q-value for policy improvement, we use the weighted combination of the maximum and the minimum Q-value of $K$ state-action value functions, which is also adopted in \bcq and \bear and shown to be useful to penalize uncertainty over future states in existing literature~\cite{fujimoto2019off}: 
\begin{equation}
\label{eq:target-value}
\begin{aligned}
y(s,a) & = r+ \gamma \max_{a_i}[\lambda \min_{j=1,\dots,K} Q_{\theta'_j}(s',a_i) \\
       & + (1-\lambda)\max_{j=1,\dots,K}Q_{\theta'_j}(\state',\action_i)]
\end{aligned}
\end{equation}
where $s'$ is the next state of current state $s$ after taking action $a$, $\theta'_{j}$ is the parameters for target network, $\gamma$ is the discount factor in Bellman Equation, $\lambda$ is the weighting factor.

Following \bear, we define the generative modeling update process as a constrained optimization problem, which tries to improve the generator and constrains  the policy within a threshold $\epsilon$:

\begin{equation}
\label{eq:gmodel-loss}
\begin{aligned}
 G_\omega = \max_{G_\omega}\EXP_{(s,a,r,s')\sim \dataset} [& \EXP_{\hat{\action} \sim \gmodel }\min_{i=1,\dots,K}Q_i(s,\hat{\action}) \\
  & - \alpha ( \loss_{\gmodel} - \epsilon)]
\end{aligned}
\end{equation}
where $\alpha$ is the tradeoff factor that balances constraining the policy to the dataset and optimizing the policy value, and can be automatically tuned via Lagrangian dual gradient descent for continuous control and is fixed as a constant for discrete control.

This forms our proposed offline RL method, which consists of three main parameterized components: a generative model $\gmodel$, Q-ensemble $\{Q_{\theta_i}\}^K_{i=1}$, target Q-networks $\{Q_{\theta'_i}\}^K_{i=1}$ and target generative model $G_{\omega'}$. 
In the following section, we demonstrate \ours results in mitigating the generative modeling error and a strong performance in the offline RL setting.

%% file: experiment.tex
\section{Experiment}

\begin{figure*}[tbh]
\centering
  \begin{tabular}{ccc}
  \includegraphics[width=.3\linewidth]{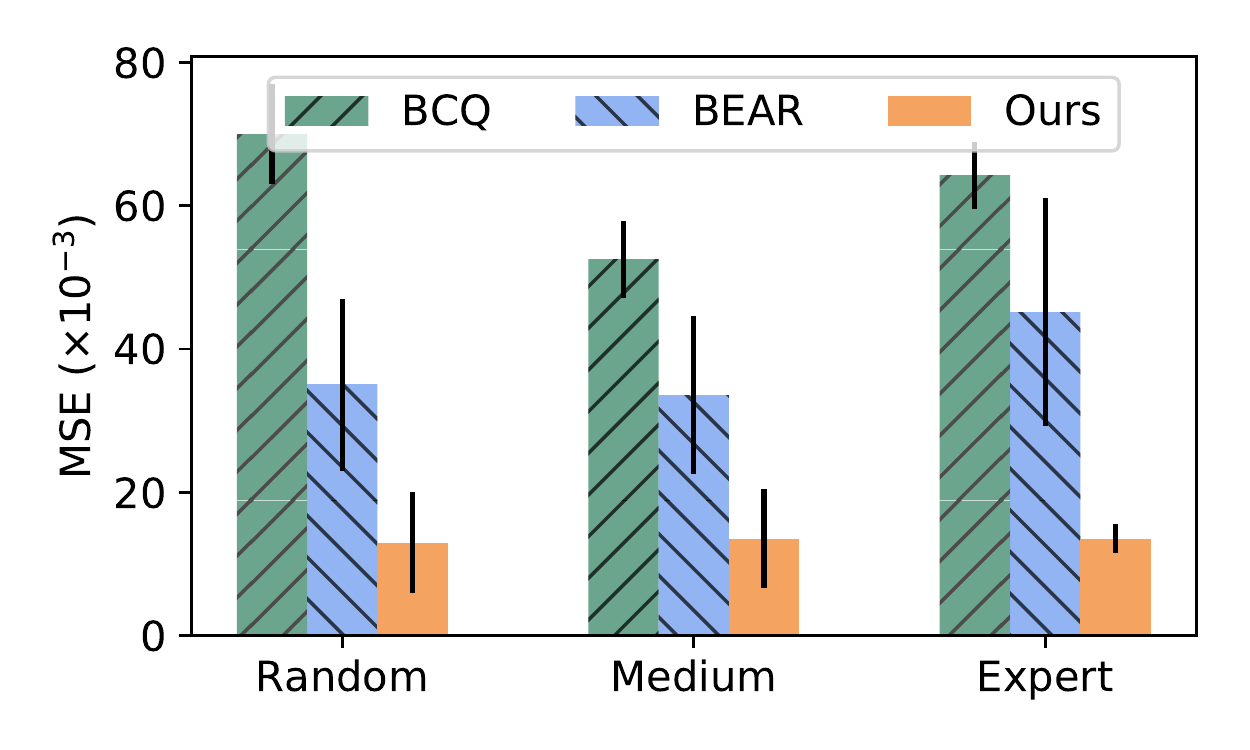} &
  \includegraphics[width=.3\linewidth]{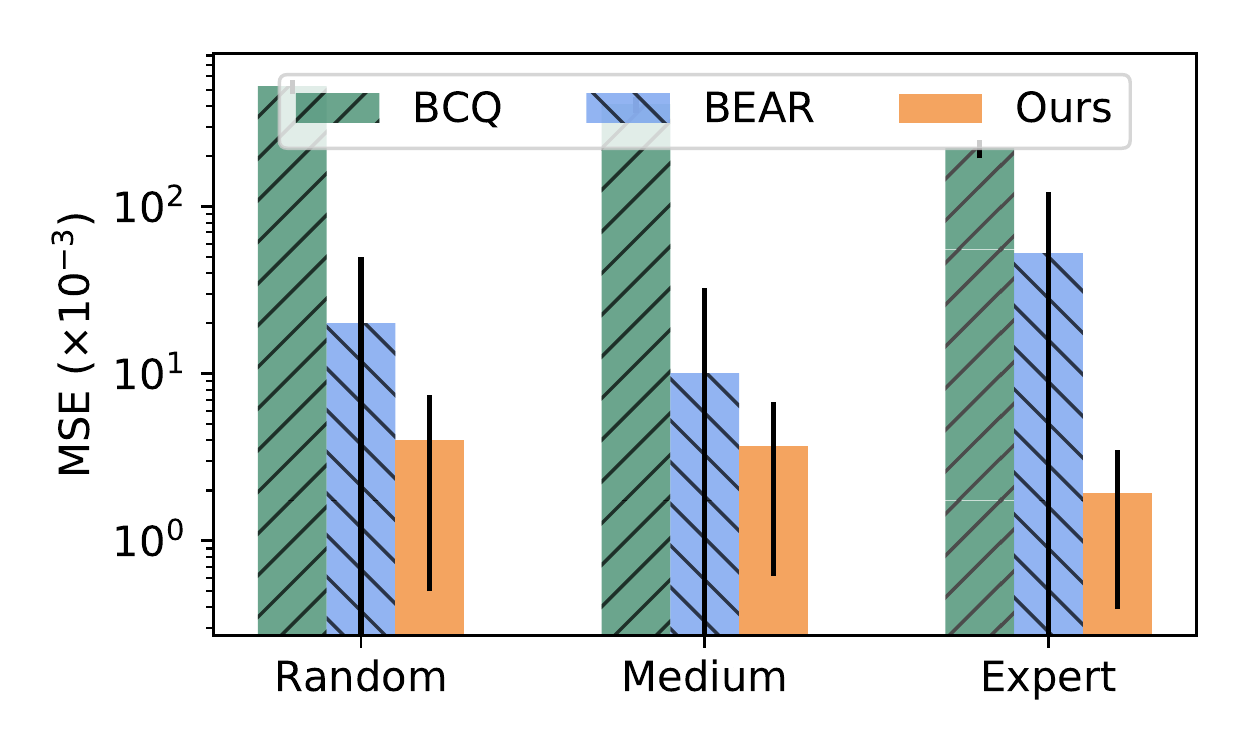} &
  \includegraphics[width=.3\linewidth]{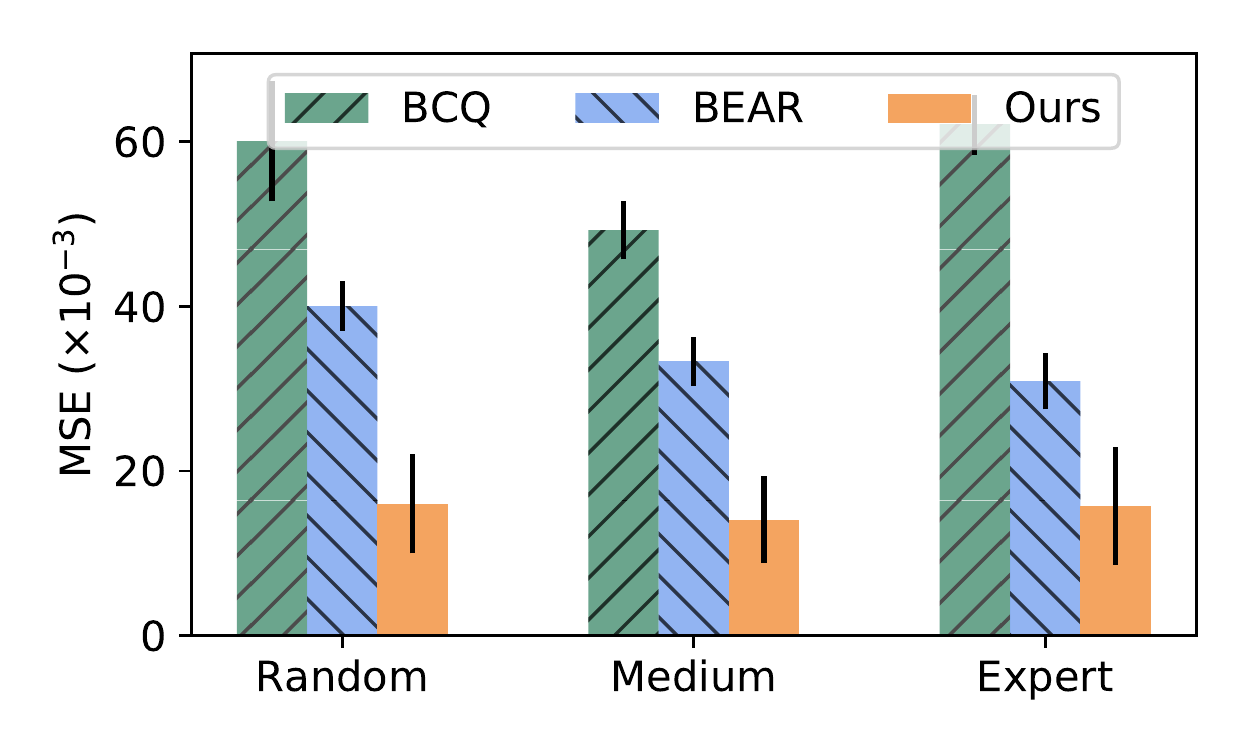}\\
  (a) \hopper & (b) \cheetah  & (b) \walker \\
  \end{tabular}
 \caption{The average (with standard deviation) of the mean squared error between $a\sim\dataset(s)$ and $\hat{\action}\sim \gmodel$ from generative model. Note that the y-axis of \cheetah~ is in log scale. The lower, the better. All the results are reported over last 500 training batches. With residual network, our proposed method can largely reduce the error.}
    \label{fig:error}
\end{figure*}

\begin{table*}[ht]
\centering
\small
\begin{tabular}{@{}ccccccccc|c@{}}
\toprule
                  & \bc                & \sac       & \ddpg    & \bcq              & \bear               & \bracp          & \bracv    & \cql     & \ours \\ \midrule
\cheetah-random   & -17.9             & \textbf{3502.0} & 209.22  & -1.3             & 2831.4             & 2713.6          & 3590.1  & \textbf{3902.4} & \textbf{3053.1}      \\
\hopper-random    & 299.4             & 347.7         & 62.13   & 323.9            & \textbf{349.9}     & 337.5           & \textbf{376.3}   & 340.7 & \textbf{379.2}       \\
\walker2d-random  & 73.0                & 192.0           & 39.1    & 228.0     & \textbf{336.3}     & -7.2            & 87.4     & \textbf{346.6}        & \textbf{350.3}       \\
\cheetah-medium   & 4196.4            & -808.6        & -745.87 & 4342.67          & 4159.08            & \textbf{5158.8} & \textbf{5473.8}  & \textbf{5236.8} & 4397.2      \\
\walker2d-medium  & 304.8             & 44.2          & 4.63    & 2441             & 2717.0      & \textbf{3559.9} & \textbf{3725.8}  & \textbf{3487.1} & 1763.5             \\
\hopper-medium    & 923.5             & 5.7           & 10.19   & \textbf{1752.4}  & 1674.5             & 1044.0            & 990.4    & \textbf{1694.0}        & \textbf{1768.1}     \\
\cheetah-expert   & \textbf{12984.5} & -230.6      & -649.1 & 10539.1         & \textbf{13130.1}  & 461.13         & -133.5     & 12189.9    & \textbf{12303.4}    \\
\hopper-expert   & \textbf{3525.4} & 22.6       & 48.2   & 3410.5          & \textbf{3567.4} & 213.5        & 119.7    & 3522.6    & \textbf{3976.4}      \\
\walker2d-expert  & \textbf{3143.9}  & -13.8      & 9.8    & \textbf{3259.2} & 3025.7            & -9.2        & 0.0     & 143.3           & \textbf{3253.5}      \\
\bottomrule
\end{tabular}
\caption{Performance of \ours and prior methods on gym domains from D4RL, on the unnormalized return metric, averaged over three random seeds, with \textbf{top-3} emphasized. While \bcq, \bear, \brac and \cql perform unstably across different scenarios, \ours consistently performs similarly or better than the best prior methods. }
\label{tab:result-all}
\end{table*}

We compare \ours to behavior cloning, off-policy methods and prior offline RL methods on a range of dataset compositions generated by (1) completely random behavior policy, (2) partially trained, medium scoring policy, and (3) an optimal policy. 
Following~\cite{fujimoto2019off,kumar2019stabilizing,wu2019behavior}, we evaluate performance on four Mujoco~\cite{todorov2012mujoco} continuous control environments in OpenAI Gym~\cite{brockman2016openai}: HalfCheetah-v2, Hopper-v2, and Walker2d-v2. We evaluate offline RL algorithms by training on these fixed datasets provided by open-access benchmarking dataset D4RL~\cite{fu2020d4rl} and evaluating the learned policies on the real environments. The statistics for these datasets can be found in~\cite{fu2020d4rl}.

\paragraph{Baselines.}
We compare with off-policy RL methods - Deep Deterministic Policy Gradient (\ddpg)~\cite{lillicrap2015continuous} and Soft Actor-Critic (\sac)~\cite{haarnoja2018soft}), Behavior Cloning (\bc)~\cite{ross2010efficient}, and state-of-the-art off-policy RL methods, including \bcq~\cite{fujimoto2019off}, \bear~\cite{kumar2019stabilizing}, Behavior Regularized Actor Critic with policy (\bracp) or value (\bracv) regularization~\cite{wu2019behavior}, and Conservative Q-Learning (\cql)~\cite{kumar2020conservative}. We use the open-source codes provided in corresponding methods and keep the same parameters in our proposed method for a fair comparison.

\paragraph{Experimental settings.}
To keep the same parameter settings as \bcq and \bear, we set $K=2$ (number of candidate Q-functions), $\lambda=0.75$ (minimum weighting factor), $\epsilon=0.05$ (policy constraint threshold), and $B=1000$ (total training steps). We report the average evaluation return over three seeds of the learned algorithm's policy, in the form of a learning curve as a function of the number of gradient steps taken by the algorithm. The samples collected during the evaluation process are only used for testing and not used for training. 

\subsection{Results}
\paragraph{Effectiveness of mitigating generative modeling error.}
In previous sections, we argued in favor of using the residual generative modeling to decrease the generative modeling error. Revisiting the argument, in this section, we investigate the empirical results on the error between true $\action$ at state $\state$ from $\dataset$ and the generated action $\hat{a}$ from $\gmodel$. In \bcq, it uses a vanilla conditional VAE as $\gmodel$; \bear use a simple feed-forward neural network. Our proposed method combines these two models with a residual network.

Figure~\ref{fig:error} shows the comparison of the error, with the mean and standard deviation from the last 500 training batches of each method. We have the following observations:
~\noindent\\$\bullet$ Both \ours and \bear have a lower error in generative modeling than \bcq, and both \ours and \bear keep the errors below $0.05$ in most cases. This is because they use the dual gradient descent to keep the target policy constrained below the threshold $\epsilon$ that is set as 0.05, while \bcq does not have this constraint.
~\noindent\\$\bullet$ Our proposed method \ours has a consistent lower error than \bear. This is because \ours uses residual learning to mitigate the generative modeling error, which matches our analysis, as suggested by Theorem~\ref{theo:residual}.

\paragraph{Effectiveness of boosting offline RL.}
As analyzed in previous sections, we can prevent an increase in the \qvalue error by learning a generative model $\gmodel$ with smaller $\errorg$, and thus learn a better offline RL policy. Here, we investigate the effectiveness of the proposed method on boosting the performance of offline RL.

Table~\ref{tab:result-all} shows the comparison of our proposed method over behavior cloning, off-policy methods, and state-of-the-art offline RL methods. We have the following observations:
~\noindent\\$\bullet$ Off-policy methods (\ddpg and \sac)  training in an purely offline setting yields a bad performance in all cases. This is due to the incorrect estimation of the value of OOD actions, which matches the existing literature.
~\noindent\\$\bullet$ Offline RL methods can outperform \bc under datasets generated by the random and medium policy. This is because \bc simply mimicking the policy behavior from the dataset without the guidance of state-action value. Since the dataset $\dataset$ is generated by a non-optimal policy, the policy learned by \bc could generate non-optimal actions. This non-optimality could accumulate as the policy rolls out in the environments.
~\noindent\\$\bullet$ \ours performs similarly or better than the best prior methods in most scenarios. We noticed that \brac and \cql yields better performance in random and medium data than in expert data. This is because under expert data, the variance of the state-action distribution in the dataset might be small, and simply mimicking the behavior policy could yield satisfactory performance (like \bc). While \brac and \cql does not have any specific design to reduce the generative error, \ours has a better generative model to mimicking the distribution of $\dataset$ and thus consistently performs as one of the best methods cross most settings. As analyzed in previous sections, we can lower the error of state-action value estimation, thus boost the offline RL performance.

We also noticed that although the overall performance of \ours is slightly better than \cql, since \cql has additional constraints that \ours does not consider. We added the constraints of \cql into \ours and found out that \ours can further improve on \cql in all cases, which means the effectiveness of boosting offline RL methods by reducing the generative error.

%% file: case.tex
\subsection{Case Study: Honor of Kings 1v1}
An open research problem for existing offline RL methods is the lack of evaluations on complex control tasks~\cite{levine2020offline,wu2019behavior} with large action spaces. Therefore, in this paper, we further implement our experiment into a multiplayer online battle arena (MOBA) game, the 1v1 version of \textit{Honor of Kings} (the most popular and widely studied MOBA game at present~\cite{ye2020towards,ye2020mastering,chen2020heroes}). Compared with traditional games like Go or Atari, \textit{Honor of Kings} 1v1 version has larger state and action space and more complex control strategies. A detailed description of the game can be found in~\cite{ye2020mastering}.

\paragraph{Baselines.} Since the action space is discrete, we compare our method with \dqn~\cite{mnih2013playing} and the discrete version of \bcq~\cite{fujimoto2019off} instead of \bear. Compared with \dqn, \bcq only adds one generative model to learn the distribution of $\dataset$. For \ours, we add an residual network upon \bcq as the discrete version of our proposed method. We use the same network parameters and training parameters for all the baseline methods for a fair comparison.

\begin{figure}[t!]
\centering
  \includegraphics[width=.95\linewidth]{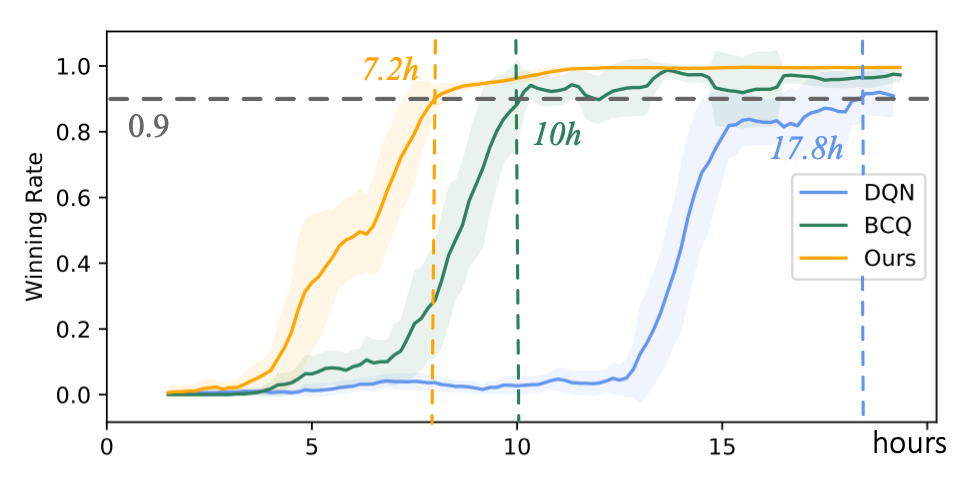}
 \caption{Winning rate of different methods against behavior-tree-based AI across training time. Our proposed method can learn a more competitive agent at a faster learning speed.}
\label{fig:case-results}
\vspace{-3mm}
\end{figure}

\paragraph{Results.} In our experiment, we evaluate the ability of the agents learned by \dqn, \bcq and our method, against the internal behavior-tree-based AI in the game, as is shown in Figure~\ref{fig:case-results}. We have the following observations:
~\noindent\\$\bullet$ Offline RL methods (\bcq and our proposed method) largely reduce the convergence time of classical off-policy method \dqn. This means the generative modeling of the dataset $\dataset$ helps the learning of value function.
~\noindent\\$\bullet$ Our proposed method further outperform \bcq. Compared with \bcq, the discrete version of our method uses the residual network in the generative modeling process, which mitigates the error of modeling the distribution of $\dataset$, and boosts the performance of learned value function.

%% file: conclusion.tex
\section{Conclusion} 

The goal of this work is to study the impact of generative modeling error in offline reinforcement learning (RL). We theoretically and empirically analyze how error propagates from generative modeling to the value function approximation in Bellman backup of off-policy RL. 
We propose \ours (action-conditioned Q-learning), a residual generative model to reduce policy approximation error.
Our experiments suggest that \ours can help to boost the performance of offline RL methods. 
Our case study on complex tasks further verifies that offline RL methods can efficiently learn with faster convergence when integrated in the online process with offline learning. 


%% file: appendix.tex
\newpage
\section{Appendix}

\subsection{Proof of Theorem~\ref{eq:qvalue-errorbound}}

\begin{proof}
Proof follows by expanding each Q, rearranging terms, simplifying the expression and then representing $\policy$ with cloned policy $\policyc$ with a generative error $\errorg$.

\begin{equation}
\begin{aligned}
  &      \error^\policy_\dataset(s,a) \\
= &  Q^*(s,a) - Q^\policy_\dataset(s,a) \\
= & \sum\limits_{s'} (p^*(s'|s, a) (r(s,a,s')+\gamma\sum\limits_{a'}\policy(a'|s') Q^*(s',a') ))\\
- & \sum\limits_{s'} (p_\dataset(s'|s, a) (r(s,a,s')+\gamma\sum\limits_{a'}\policy(a'|s') Q^\policy_\dataset(s',a') )) \\
= & \sum\limits_{s'} (p^*(s'|s, a) - p_\dataset(s'|s, a))r(s,a,s') \\
& + p^*(s'|s, a) \gamma\sum\limits_{a'}\policy(a'|s') Q^*(s',a') \\
& -  p_\dataset(s'|s, a)\gamma\sum\limits_{a'}\policy(a'|s') Q^\policy_\dataset(s',a') \\
= & \sum\limits_{s'} (p^*(s'|s, a) - p_\dataset(s'|s, a))r(s,a,s') \\
& + p^*(s'|s, a) \gamma\sum\limits_{a'}\policy(a'|s') (Q^\policy_\dataset(s',a')+\error^\policy_\dataset(s',a')) \\
& -  p_\dataset(s'|s, a)\gamma\sum\limits_{a'}\policy(a'|s') Q^\policy_\dataset(s',a') \\
= &  \sum\limits_{s'} (p^*(s'|s, a) - p_\dataset(s'|s, a)) \\
    & \cdot [r(s,a,s')+\gamma\sum\limits_{a'}\policy(a'|s') Q^\policy_\dataset(s',a') ] \\
 & +  p^*(s'|s, a)\cdot \gamma \sum\limits_{a'}\policy(a'|s')\error^\policy_\dataset(s',a') \\
= &  \sum\limits_{s'} (p^*(s'|s, a) - p_\dataset(s'|s, a)) \\
    & \cdot [r(s,a,s')+\gamma\sum\limits_{a'}(\policyb(a'|s') + \errorg(a'|s')) Q^\policy_\dataset(s',a') ] \\
 & + p^*(s'|s, a)\cdot \gamma \sum\limits_{a'}(\policyb(a'|s') + \errorg(a'|s'))\error^\policy_\dataset(s',a')
\end{aligned}
\end{equation}

\end{proof}


\subsection{Training procedure}
\begin{algorithm}[t!]
\small
\DontPrintSemicolon
\caption{Training procedure of \ours}
\label{alg:ours}

\KwIn{Dataset $\dataset$, target network update rate $\tau$, total batches $B$, number of sampled actions $n$, minimum $\lambda$}
\KwOut{Generative model $G_\omega$, Q-ensemble $\{Q_{\theta_i}\}^K_{i=1}$, target Q-networks $\{Q_{\theta'_i}\}^K_{i=1}$ and target generative model $G_{\omega'}$}

Initialize Q-ensemble $\{Q_{\theta_i}\}^K_{i=1}$, generative model $\gmodel_\omega$, Lagrange multiplier $\alpha$, target networks $\{Q_{\theta'_i}\}^K_{i=1}$ and target generative model $\gmodel_{\omega'}$ with $\omega' \leftarrow \omega$, $\theta' \leftarrow \theta$ ; \;
\For{i $\longleftarrow$ 0, 1, \dots, $B$}
{
    Sample mini-batch of transitions $(\state,\action,\reward,\state') \sim \dataset$ ; \;
    \nonl
    \textbf{Value function approximation update:} \;
    Sample $p$ action samples, $\{\action_i \sim G_{\omega'}(\cdot|\state') \}^p_{i=1}$ ; \;
    Compute target value $y(s,a)$ using Equation~\eqref{eq:target-value} ; \;
    $\forall i$, $\theta_i \leftarrow  \arg \min_{\theta_i} (Q_{\theta_i}(\state,\action)-  y(s,a))^2$ ; \;
    \nonl
    \textbf{Generative modeling update:} \;
    Sample actions $\{ \hat{\action}_j\sim \gmodel \}^n_{j=1}$ and $\{\action_j \sim \dataset(\state) \}^n_{j=1}$ ; \;
    Update $\omega$, $\alpha$ by minimizing the objective function  by using dual gradient descent ; \;
    
    Update target networks: $\omega'_i \leftarrow \tau \omega_i + (1-\tau) \omega'_i$, $\theta'_i \leftarrow \tau \theta_i + (1-\tau) \theta'_i $ ; \;
    
}
\end{algorithm}

\subsection{Mujoco Experiment Parameter Settings}

For the Mujoco tasks, we build \ours on top of the implementation of \bear, which was provided at in~\cite{fu2020d4rl}. Following the convention set by~\cite{fu2020d4rl}, we report the unnormalized, smooth average undiscounted return over 3 seed for our results in our experiment.

The other hyperparameters we evaluated on during our experiments, and might be helpful for using \ours are as follows:

\begin{itemize}
    \item Network parameters. For all the MuJoCo experiments, unless specified, we use fully connected neural networks with ReLU activations. For policy networks, we use tanh (Gaussian) on outputs following \bear, and all VAEs are following the open sourced implementation of \bcq. For network sizes, we shrink the policy networks from (400, 300) to (200,200) for all networks, including \bcq and \bear for fair comparison and saving computation time without losing performance. We use Adam for all optimizers. The batch size is 256 for all methods except for BCQ, where in the open sourced implementation of BCQ, it is 100 and we keep using 100 in our experiments.
    \item Deep RL parameters. The discount factor $\gamma$ is always 0.99. Target update rate is 0.05. At test time we follow \bcq and \bear by sampling 10 actions from $\policy$ at each step and take one with the highest learned Q-value.
\end{itemize}

\subsection{Additional improvements over existing offline RL methods}
We also acknowledge that recently there are various of offline RL methods proposed to reduce the value function approximation error. In this part, we show our method can improve existing offline RL methods by additionally reducing the generative modeling error. Here, we take a most recent method in continuous control scenario, \cql, and compare it with the variant of our method:
~\noindent\\$\bullet$ \textbf{\ours-C} is based on \ours, which additionally minimizes Q-values of unseen actions during the value function approximation update process. This penalty is inspired by \cql, which argues the importance of conservative off-policy evaluation.

\begin{table*}[htb]
\centering
\caption{Additional improvements over \cql. Performance of \ours on the normalized return metric, averaged over 3 random seeds. }
\label{tab:additional-improvements}
\begin{tabular}{cccc}
\toprule
                   & AQL         & CQL           & AQL-C         \\
                   \midrule
\cheetah-random  & 25.16 (8.31)   & 32.16 (3.31)   & 36.48 (4.1)    \\
\hopper-random     & 11.72 (0.81)   & 10.53 (4.12)   & 12.24 (1.07)   \\
\walker-random   & 7.63 (1.28)    & 7.55 (0.81)    & 8.47 (0.93)    \\
\cheetah-medium   & 36.24 (2.47)   & 43.16 (4.42)   & 44.13 (3.14)   \\
\walker-medium   & 38.4 (9.67)    & 75.93 (7.7)    & 76.16 (9.7)    \\
\hopper-medium    & 54.67 (13.73)  & 52.38 (5.35)   & 55.68 (6.35)   \\
\cheetah-expert  & 101.39 (10.22) & 100.38 (10.17) & 102.46 (11.21) \\
\hopper-expert    & 122.94 (8.28)  & 108.91 (11.03) & 110.76 (9.16)  \\
\walker-expert   & 70.85 (4.8)    & 3.12 (2.44)    & 70.15 (0.48)   \\
\bottomrule
\end{tabular}
\end{table*}

\begin{figure}[t!]
\centering
  \includegraphics[width=0.45\textwidth]{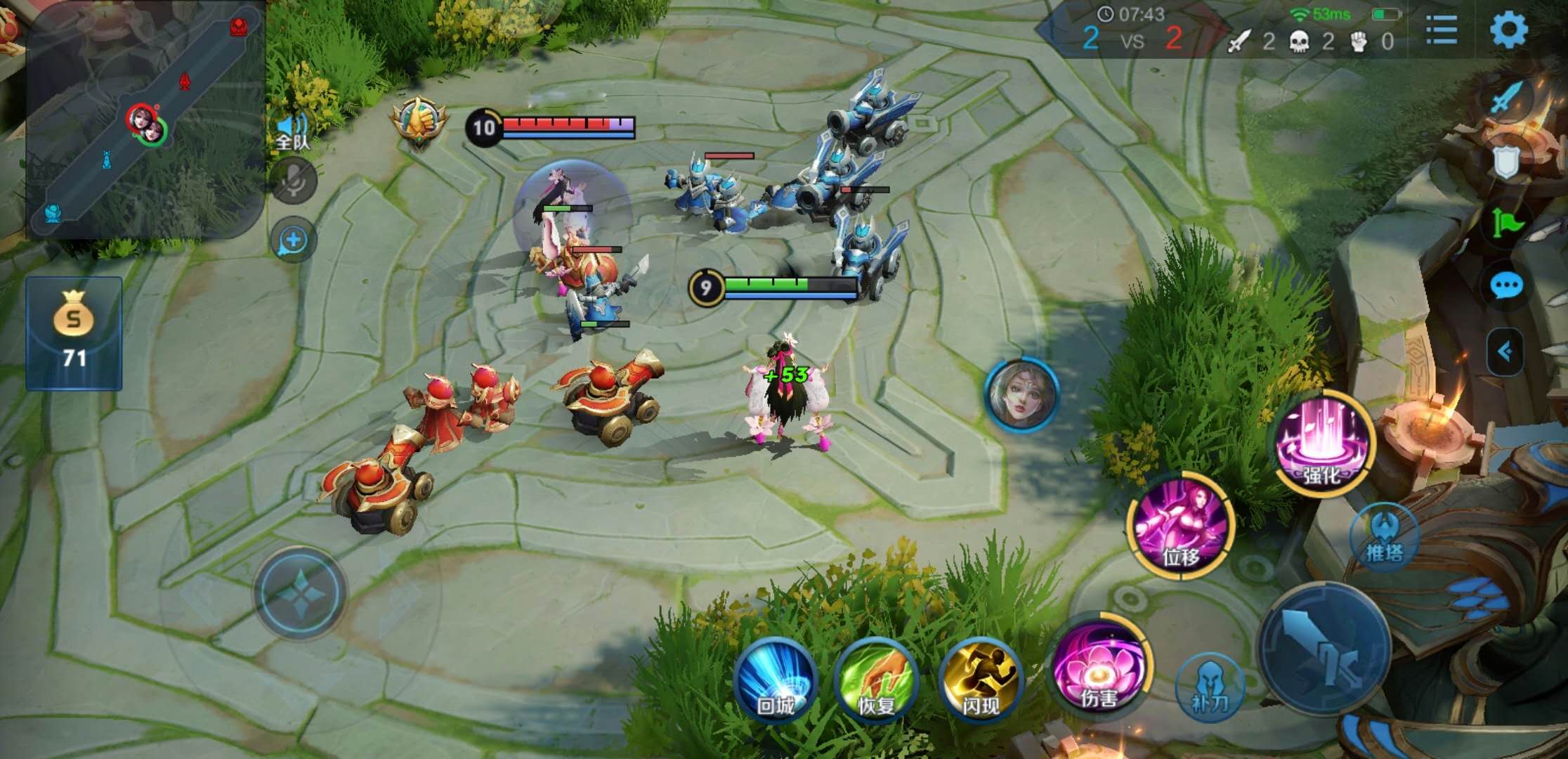}
 \caption{The environment of Honor of Kings}
    \label{fig:environment}
\end{figure}

\begin{figure}[t!]
\centering
  \includegraphics[width=0.45\textwidth]{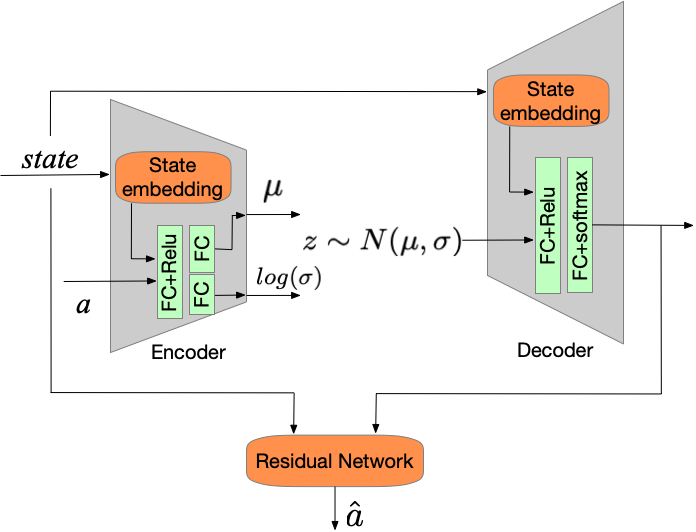}
 \caption{Proposed residual generative model}
    \label{fig:model}
\end{figure}

\subsection{Honor of Kings Game Description}
The Honor of Kings 1v1 game environment is shown in Figure~\ref{fig:environment}.
In the Honor of Kings 1v1 game, there are two competing agents, each control one hero that can gain golds and experience by killing creeps, heroes or overcoming towers. The goal of an agent is to destroy its opponent's crystal guarded by the tower. The state of an agent is considered to be a 2823-dimensional vector containing information of each frame received from game core, \eg hero health points, hero magic points, location of creeps, location of towers, \etc. The action of an agent is a 79-dimensional one-hot vector, indicating the directions for moving, attacking, healing, or skill releasing. Compared with traditional games like Go or Atari, \textit{Honor of Kings} 1v1 version has larger state and action space and more complex control strategies, as is indicated in~\cite{ye2020mastering}.


\subsection{Honor of Kings Parameter Settings}

\paragraph{Experimental Settings} In our experiment, we aim to learn to control \textit{Diao Chan}, a hero in Honor of Kings. Specifically, we are interested in whether the offline RL method can accelerate the learning process of existing online off-policy RL. The training procedure is an iterative process of off-policy data collection and policy network updating. During the off-policy data collection process, we run the policy over 210 CPU cores in parallel via self-play with mirrored policies to generate samples with $\epsilon$-greedy algorithm. All the data samples are then collected to a replay buffer, where the samples are utilized as offline data in the policy network updating process. We train our neural network on one GPU by batch samples with batch size 4096. We test each model after every iteration against common AI for 60 rounds in parallel to measure the capability of the learned model.

\begin{itemize}
    \item Network parameters. For all the Honor of Kings experiments, unless specified, we use fully connected neural networks with ReLU activation. We use dueling \dqn with state value network and advantage network, with shared layers of sizes (1024, 512) and separate layers of sizes (512, ) and (512, ). We use linear activation in state value network and advantage network. Based on \dqn, \bcq has an additional feed-forward network with the size of hidden layers as (1024, 512, 512). \ours has the same feed-forward network as \bcq, and a conditional-VAE whose architecture as \bear, which is shown in Figure~\ref{fig:model}: the state embedding module has two layers with the size of (1024, 512). The latent vector $z$ has the same dimension the action space. 
    \item RL parameters. The discount factor $\gamma$ is always 0.99. Target update rate is 0.05. At test time we follow \bcq and \bear by sampling 10 actions from $\policy$ at each step and take one with the highest learned Q-value. For \bcq, the policy filters actions that has lower probabilities over the highest Q-value with threshold $\epsilon=0.1$. For \ours,  the tradeoff factor $\alpha=0.05$ and the $\epsilon=0.1$.
\end{itemize}

